\documentclass[a4paper]{article}

\usepackage[serbian,english]{babel}
\usepackage{fullpage}
\usepackage{microtype}
\usepackage{latexsym}
\usepackage{xspace}
\usepackage{amsthm}
\usepackage{amssymb}
\usepackage{amsmath}
\usepackage{mathtools}
\usepackage{amssymb}
\usepackage{booktabs}
\usepackage{paralist}
\usepackage{multirow}
\usepackage{nicefrac}
\usepackage{tikz}
\usepackage{todonotes}
\usepackage{MnSymbol}
\usepackage{etoolbox}
\usepackage{url}
\usepackage{graphicx}
\usepackage{multibib}
\usepackage{stmaryrd}
\usepackage{dsfont}

\mathchardef\mhyp="2D

\newtheorem{lemma}{Lemma}

\newtheorem{proposition}{Proposition}
\newtheorem{definition}{Definition}
\newtheorem{theorem}{Theorem}
\newtheorem{remark}{Remark}
\newtheorem{claim}{Claim}
\newtheorem{example}{Example}
\newtheorem{corollary}{Corollary}

\newcommand{\game}{P}
\newcommand{\pomdp}{P}

\newcommand{\states}{S}

\newcommand{\act}{\mathcal{A}}
\newcommand{\trans}{\delta}
\newcommand{\obs}{\mathcal{Z}}
\newcommand{\reward}{r}
\newcommand{\obsmap}{\mathcal{O}}
\newcommand{\distr}{\mathcal{D}}
\newcommand{\initd}{\lambda}

\newcommand{\supp}{\mathrm{Supp}}

\newcommand{\reals}{\mathbb{R}}

\newcommand{\belief}{b}
\newcommand{\discsum}{\mathsf{Disc}}
\newcommand{\discount}{\gamma}

\newcommand{\obsfunc}{\obsmap}
\newcommand{\hist}{h}
\newcommand{\ob}{o}

\newcommand{\histfunc}{H}

\newcommand{\E}{\mathbb{E}}
\newcommand{\evalue}{\mathit{eVal}}
\newcommand{\wvalue}{\mathit{wVal}}
\newcommand{\thr}{t}
\newcommand{\Rset}{\mathbb{R}}

\newcommand{\len}[1]{\mathit{len}(#1)}
\newcommand{\belseq}[2]{#1_{#2}}
\newcommand{\discpath}[2]{\def\EmptyTest{#1}\ifdefempty{\EmptyTest}{\discsum_{#2}}{\discsum_{#2}(#1)}}
\newcommand{\allowed}[2]{\mathit{Allow_{#2}^{#1}}}
\newcommand{\belsup}{B}

\newcommand{\valbound}{\Psi}
\newcommand{\fvalue}{\mathit{fVal}}
\newcommand{\succb}{\Delta}
\newcommand{\ValBelSup}{\mathit{VBelSup}}
\newcommand{\gvalue}{\mathit{gVal}}

\newcommand{\remain}[3]{\mathit{rem}^{#3}_{#2}(#1)}

\newcommand{\myparagraph}[1]{\paragraph{#1}}

\usepackage{nicefrac}
\usepackage{tikz}
\usetikzlibrary{fit}
\usetikzlibrary{backgrounds}
\usetikzlibrary{automata}
\usetikzlibrary{shapes}
\usetikzlibrary{matrix}
\usetikzlibrary{fit}
\usetikzlibrary{calc}
\usetikzlibrary{positioning}
\usetikzlibrary{intersections}

\tikzstyle{Player1}=[circle, thick, minimum size=0.6cm, inner 
sep=0cm,draw=black]
\tikzstyle{State}=[circle, font = \small, thick, minimum size=0.6cm, inner 
sep=0cm,draw=black,fill=white]
\tikzstyle{Final}=[circle, accepting, thick, minimum size=0.6cm, inner 
sep=0cm,draw=black]
\tikzstyle{RState}=[circle, very thick, minimum size=0.8cm, inner 
sep=0cm,draw=red]
\tikzstyle{tran}=[draw,->,font=\small]
\tikzstyle{obstyle}=[rounded corners,fill=gray!20]

\newcites{trsec}{References}

\title{Optimizing Expectation with Guarantees in POMDPs\\(Technical Report)}
\author{Krishnendu Chatterjee, Petr Novotn\'{y}\\
	IST Austria, Klosterneuburg, Austria\\
	\texttt{krishnendu.chatterjee@ist.ac.at, pnovotny@ist.ac.at}
	\and Guillermo A. P\'{e}rez\thanks{Author supported by an F.R.S.-FNRS
	Aspirant fellowship.}, Jean-Fran\c{c}ois Raskin\\
	Universit\'{e} Libre de Bruxelles, Brussels, Belgium\\
	\texttt{jraskin@ulb.ac.be, gperezme@ulb.ac.be}
	\and \foreignlanguage{serbian}{\DJ{}or\dj{}e \v{Z}ikeli\'c}\\
	University of Cambridge, Cambridge, UK\\
	\texttt{dz277@cam.ac.uk}}

\begin{document}

\maketitle

\begin{abstract}
A standard objective in partially-observable Markov decision processes (POMDPs)
is to find a policy that maximizes the expected discounted-sum payoff. However,
such policies may still permit unlikely but highly undesirable outcomes, which
is problematic especially in safety-critical applications. Recently, there has
been a surge of interest in POMDPs where the goal is to maximize the probability
to ensure that the payoff is at least a given threshold, but these approaches do
not consider any optimization beyond satisfying this threshold constraint. In
this work we go beyond both the ``expectation'' and ``threshold'' approaches and
consider a ``guaranteed payoff optimization (GPO)'' problem for POMDPs, where we
are given a threshold $t$ and the objective is to find a policy $\sigma$ such
that a) each possible outcome of $\sigma$ yields a discounted-sum payoff of at
least $t$, and b) the expected discounted-sum payoff of $\sigma$ is optimal (or
near-optimal) among all policies satisfying a). We present a practical approach
to tackle the GPO problem and evaluate it on standard POMDP benchmarks. 
\end{abstract}

\section{Introduction}
The \textit{de facto} model for decision making under uncertainty are
partially-observable Markov decision processes
(POMDPs)~\cite{LittmanThesis,PT87}, and they have been applied in diverse
applications ranging from planning~\cite{RN10}, to reinforcement
learning~\cite{LearningSurvey}, to robotics~\cite{KGFP09,kaelbling1998planning}.
One of the classical and fundamental payoff function for POMDPs is the {\em
discounted-sum payoff} that aggregates the rewards of the transitions as a
discounted sum. The traditional objective in POMDPs has been to obtain  policies
that maximize the expected discounted-sum payoff.

One crucial drawback of the traditional objective (that asks for expectation
maximization) is that it allows for undesirable events that can happen with low
probability.  For example, consider a policy $\sigma_1$ that with probability
$1/2$ achieves payoff $100$ and with probability $1/2$ achieves payoff $0$, and
a different policy $\sigma_2$ that achieves payoff~$20$ with probability~$1$.
If payoff values below~$10$ are undesirable, then the first policy, though
better for expected payoff, allows undesirable events with significant
probability, and hence the second policy is preferable.  Hence, there has been a
recent interest to study objectives where, instead of maximizing the expected
payoff~\cite{HYV16:risk-pomdps}, the goal is to maximize the probability that 
the payoff is above a threshold. 

A drawback of the approach to maximize the probability that the payoff exceeds a
threshold is that it ignores the optimization aspect of maximizing the
expectation.  In this work we consider an objective for POMDPs where both
aspects are present.  More precisely, we consider a ``guaranteed payoff
optimization (GPO)'' problem for POMDPs, where given a threshold $t$, the goal
is to maximize the expectation while ensuring that the payoff is at least $t$.

As a concrete motivation for the GPO problem, consider planning under 
uncertainty (e.g., self-driving cars) where certain events are catastrophic 
(e.g., crashes), and in the model they are assigned low payoffs. 
Such catastrophic events must be avoided even at the expense of 
expected payoff. That is, policies must maximize the expected payoff, 
ensuring the avoidance of catastrophic events. 
Hence, for planning in safety-critical applications the GPO problem 
is natural.

In this work, our main contributions are as follows:
\begin{compactenum}
\item We study the GPO problem for POMDPs, and present a practical solution
	approach for the problem.  In particular, given a POMDP with the GPO
	problem, we present a transformation to a different POMDP where it
	suffices to solve the traditional expectation objective.  Our solution
	approach first constructs a representation of all strategies that
	satisfies item a) of the GPO problem, and then we extend the
	partially-observable Monte Carlo planning (POMCP) approach to obtain
	optimal policies w.r.t. expectation among the above strategies. 
\item We present experimental results on several classical POMDP examples from
	the literature to show how our approach can efficiently  solve the GPO
	problem for POMDPs.
\end{compactenum}

\myparagraph{Related Works.}
Works studying POMDPs with discounted sum
range from theoretical results (see, e.g.,~\cite{PT87,LittmanThesis}) 
to practical tools (e.g.~\cite{khl08,SV:POMCP}). 
Recent works focus on extracting policies which ensure that,
with a given probability bound, the obtained discounted-sum payoff is above a
threshold (see, e.g.,~\cite{HYV16:risk-pomdps}). The problem of ensuring the 
payoff is above a
given threshold while optimizing the expectation has been considered for
fully-observable MDPs and the long-run average and stochastic shortest path 
objectives~\cite{bfrr14,rrs15}; and also with probabilistic thresholds for
long-run average payoff~\cite{ckk15}. As for POMDPs, we mention 
\emph{constrained 
POMDPs}~\cite{UH10:constrained-pomdp-online,PMPKGB15:constrained-POMDP}, where 
the aim is to maximize the expected payoff
while ensuring that the expectation of some other quantity is bounded. 
In contrast, our constraints are \emph{hard}, i.e. they must hold always, not
just on average.  
The work probably closest to ours is~\cite{STW16:BWC-POMDP-state-safety} that 
also considers maximizing expected payoff among all policies satisfying a 
given constraint, but there are two key differences from our work: 
they consider finite horizon POMDPs, while we consider infinite horizon ones, 
and more importantly, their constraints are \emph{state-based}, i.e. their 
policy must ensure that the execution of the POMDP does not go through certain
``violating'' states.  In contrast, our ``threshold constraint'' is
\emph{execution-based}: whether a execution yields payoff at least $t$ cannot be
determined solely by looking at the set of states appearing in the execution,
but the whole infinite execution has to be considered. This requires very
different techniques.  To our best knowledge, the GPO problem has never been
considered for POMDPs with discounted sum.

\section{Preliminaries}
\label{sec:prelims}
Throughout this work, we follow standard (PO)MDP notations 
from~\cite{Puterman2005,LittmanThesis}.

\subsection{POMDPs}
We denote by $\distr(X)$ the set of all probability distributions on a finite 
set $X$, i.e. all functions $f: X \rightarrow [0,1]$ such that $\sum_{x\in 
X}f(x)=1$. For $f\in \distr(X)$ we denote by $\supp(f)$ the \emph{support} of 
$f$, i.e. the set $\{x\in X\mid f(x)>0\}$.

\begin{definition}\textbf{POMDPs.}
A \emph{POMDP} is defined as a
tuple $\pomdp=(\states,\act,\trans,\reward,\obs,\obsmap,\initd)$ 
where
$\states$ is a finite set of \emph{states},
$\act$ is a finite alphabet of \emph{actions},
$\trans:S\times\act \rightarrow \distr(\states)$ is a 
 \emph{probabilistic transition function} that given a state $s$ and an
 action $a \in \act$ gives the probability distribution over the successor 
 states, 
$\reward: \states \times \act \rightarrow \reals$  is a reward 
 function,
$\obs$ is a finite set of \emph{observations},
$\obsmap:\states\rightarrow \distr(\obs)$ is a probabilistic 
 \emph{observation function} that 
  maps every state to a distribution over observations, and 
$\initd\in \distr(\states)$ is the \emph{initial belief}.
We abbreviate $\trans(s,a)(s')$ by 
 $\trans(s'|s,a)$,
\end{definition}

\begin{remark}[Deterministic observation function]
\label{rem:det-obs}
Deterministic observation functions of type $\obsfunc : S 
\rightarrow \obs$ are sufficient in POMDPs (see Remark $1$ in~\cite{CCGK14a}).
Informally, the probabilistic aspect of the observation function can be 
encoded into
the 
transition function and, by letting the product of the states and observations
be the new state-space,
we obtain a deterministic observation function.
Thus, without loss of generality, we will always consider 
observation functions of type $\obsfunc : S \rightarrow \obs$,
which greatly simplifies the notation.
\end{remark}

\myparagraph{Plays \& Histories.}
A \emph{play} (or an infinite path) in a POMDP is an infinite sequence $\rho = 
s_0 a_0 s_1 a_1 s_2 a_2 \ldots$ 
of states and actions such that
$s_0 \in \supp(\initd)$ and
for all $i \geq 0$ 
we have $\trans(s_i,a_i)(s_{i+1})>0$. We write $\Omega$ for the set of all 
plays. A \emph{finite path} (or just \emph{path}) is a finite prefix of a 
play ending with a state, i.e. a sequence from $(\states\cdot\act)^*\cdot 
\states$. 
A~\emph{history} is 
a finite sequence of actions and observations 
$\hist=a_0 \ob_1 \dots a_{i-1} \ob_i\in (\act\cdot\obs)^*$ 
such that there is a path $w=s_0 a_0 s_1 \dots a_{i-1} s_i$ with 
$\ob_j=\obsfunc(s_j)$
for each $1\leq j \leq i$.
We write $\hist=\histfunc(w)$ to indicate that history 
$\hist$ 
corresponds to a path $w$. The \emph{length} of a path (or history) $w$,
denoted by $\len{w}$, 
is the number of actions in $w$, and the length of a play $\rho$ is 
$\len{\rho}=\infty$.

\myparagraph{Beliefs.}
A \emph{belief} is a distribution on states (i.e. an element of $\distr(\states)$)
indicating the probability of being in each particular state given the current
history. The initial belief $\initd$ is given as part of the POMDP.
Then, in each step, when the history observed so far is $h$, the current belief
is $\belseq{\belief}{h}$, an action $a\in \act$ is played and an observation
$z\in \obs$ is received, the updated belief $\belseq{\belief}{h'}$ for history
$h'=hao$ can be computed by a standard formula~\cite{cassandra1998exact}.

\myparagraph{Infinite-horizon Discounted Payoff.}
Given a play $\rho =s_0 a_0 s_1 a_1 s_2 a_2 \ldots$ and a discount 
factor $0 
\leq \discount < 1$, the \emph{infinite-horizon discounted payoff} 
$\discsum_\discount$ of $\rho$ is:
\[ \textstyle
	\discpath{\rho}{\discount} = \sum_{i=0}^{\infty} 
	\discount^{i}\reward(s_i,a_i).
\]
We also define a discounted payoff of a finite path $w$ as 
$\discpath{w}{\discount}=\sum_{i=0}^{\len{w}-1} 
	\discount^{i}\reward(s_i,a_i).$

\myparagraph{Policies.}
A \emph{policy} is a blueprint for selecting actions based on 
the past history of observations and actions. 
Formally, it 
is a function $\sigma$ 
which assigns to a history a probability distribution 
over the actions, i.e. $\sigma(\hist)(a)$ is the probability of selecting 
action $a$ after observing history $\hist$ (we often abbreviate 
$\sigma(\hist)(a)$ to 
$\sigma(a\mid\hist)$).

\myparagraph{Consistent Plays.}
A play or a path $w$ is \emph{consistent} with a policy $\sigma$ if it can be
obtained by extending its finite prefixes using $\sigma$.  Formally, $w=s_0 a_0
s_1 a_1 \dots$ is consistent with $\sigma$ if for each $0\leq i \leq \len{w}$
there is action $a$ such that  $\sigma(a\mid \histfunc(s_0 a_0 \dots a_{i-1} s_i))>0$
and $\trans(s_{i+1}\mid s_i,a)>0$. A history $h$ is consistent with $\sigma$ if
there is a path $w$ consistent with $\sigma$ such that $h=H(w)$.

\myparagraph{Expected Value $\evalue^\pomdp$ of Policies.}
Given a POMDP $\pomdp$, a policy $\sigma$, a discount factor $\discount$, and 
an initial belief 
$\initd$,
the \emph{expected value}
of $\sigma$ from $\initd$ is the expected 
value of the 
infinite-horizon discounted sum 
under policy $\sigma$ when starting in a state sampled from $\initd$:
\(
	\evalue^\pomdp(\sigma) = \E_\initd^{\sigma}[\discpath{}{\discount}].
\)
This definition can be formalized by a standard construction of a probability 
measure induced by $\sigma$ over the set of all plays, which also gives rise to 
the expectation operator $\E^\sigma_\initd$~(see, e.g.,~\cite{Puterman2005}).

\myparagraph{Worst-Case Value $\wvalue^{\pomdp}$ of Policies.}
The \emph{worst-case} value of a policy $\sigma$ from belief $\initd$ is 
\( \textstyle
\wvalue^{\pomdp}(\sigma)=\inf_{\rho} \discpath{\rho}{\discount},
\)
where the infimum is taken over the set of \emph{all plays that are consistent 
with 
$\sigma$} and start in a state sampled from $\initd$.

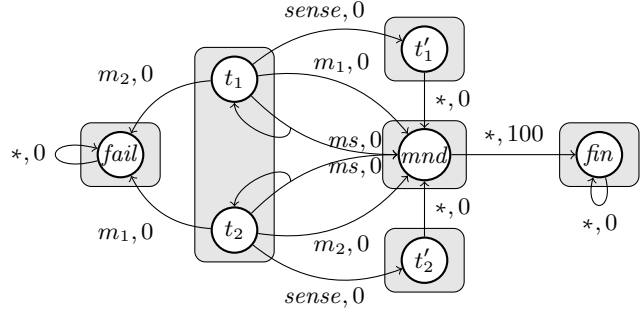
\begin{figure}
\centering
\begin{tikzpicture}
\newlength{\obsboundingbox}
\setlength{\obsboundingbox}{2mm}

\node[State] (t1) at (-2,1) {$t_1$};
\node[State] (t1o) at (0.5,1.4) {$t_1'$};
\node[State] (t2) at (-2,-1) {$t_2$};
\node[State] (t2o) at (0.5,-1.4) {$t_2'$};
\node[State] (tgt) at (0.5,0) {$\mathit{mnd}$};
\node[State] (bh) at (-3.5,0) {$\mathit{fail}$};
\node[State] (fin) at (2.8,0) {$\mathit{fin}$};

\draw[tran, loop left, looseness =12] (bh) to node {$*,0$} (bh);
\draw[tran, bend right] (t1) to node[auto,swap] {$m_2,0$} (bh);
\draw[tran, bend left] (t2) to node[auto] {$m_1,0$} (bh);
\draw[tran, bend left] (t1) to node[above] {$\mathit{sense},0$} (t1o);
\draw[tran, bend right] (t2) to node[below] {$\mathit{sense},0$} (t2o);
\draw[tran] (tgt) to node[above] {$*,100$} (fin);
\draw[tran] (t1o) to node[right] {$*,0$} (tgt);
\draw[tran] (t2o) to node[right] {$*,0$} (tgt);
\draw[tran, bend left] (t1) to node[above] {$m_1,0$} (tgt);
\draw[tran, bend right] (t2) to node[below] {$m_2,0$} (tgt);

\draw[tran, out = -45, in=180] (t1) to 
node[pos=0.3,coordinate,name=help1,auto] 
{} node[pos=0.15,font=\small,label={[font=\small,label distance = 
	-0mm]0:{$\mathit{ms},0$}}] {} (tgt);
\path[draw,tran,out=-89,in=-90] (help1) to (t1);
\draw[tran, in=180] (t2) to node[pos=0.3,coordinate,name=help2,auto] 
{} node[pos=0.15,font=\small,label={[font=\small,label distance = 
	-0mm]0:{$\mathit{ms},0$}}] {} (tgt);
\path[draw,tran,out=89,in=90] (help2) to (t2);
\draw[tran, loop below] (fin) to node {$*,0$} (fin);

\begin{pgfonlayer}{background}
      \node[coordinate,above left = \obsboundingbox and \obsboundingbox+1mm of
      t1] (dum1) {};
      \node[coordinate, below right = \obsboundingbox and \obsboundingbox+1mm of
      t2] (dum2) {};
      \draw[obstyle] (dum1) rectangle (dum2);

      \node[coordinate,above left = \obsboundingbox and \obsboundingbox+1mm of
      tgt] (dum3) {};
      \node[coordinate, below right = \obsboundingbox and \obsboundingbox+1mm 
of
      tgt] (dum4) {};
      \draw[obstyle] (dum3) rectangle (dum4);
      
      \node[coordinate,above left = \obsboundingbox and \obsboundingbox+1mm of
            t1o] (dum5) {};
            \node[coordinate, below right = \obsboundingbox and 
            \obsboundingbox+1mm 
      of
            t1o] (dum6) {};
            \draw[obstyle] (dum5) rectangle (dum6);
            
\node[coordinate,above left = \obsboundingbox and \obsboundingbox+1mm of
      t2o] (dum7) {};
      \node[coordinate, below right = \obsboundingbox and \obsboundingbox+1mm 
of
      t2o] (dum8) {};
      \draw[obstyle] (dum7) rectangle (dum8);

\node[coordinate,above left = \obsboundingbox and \obsboundingbox+1mm of
      bh] (dum9) {};
      \node[coordinate, below right = \obsboundingbox and \obsboundingbox+1mm 
of
      bh] (dum10) {};
      \draw[obstyle] (dum9) rectangle (dum10);
      
\node[coordinate,above left = \obsboundingbox and \obsboundingbox+1mm of
fin] (dum11) {};
\node[coordinate, below right = \obsboundingbox and \obsboundingbox+1mm 
of
fin] (dum12) {};
\draw[obstyle] (dum11) rectangle (dum12);
\end{pgfonlayer}

\end{tikzpicture}
\caption{Illustrative POMDP. We assume a discount factor
	$\discount=\frac{1}{2}$.  Gray rectangles represent observations. The
	only probabilistic branching occurs when $\mathit{ms}$ is played in
	$t_1$ or $t_2$, and for both $i\in\{1,2\}$ we have
	$\trans(\mathit{mnd}\mid t_i,\mathit{ms})=\frac{3}{5} $ and $\trans(t_i
	\mid t_i,\mathit{ms})=\frac{2}{5}$. The initial belief $\initd$ assigns
	$\frac{9}{10}$ to state $t_1$ and $\frac{1}{10}$ to $t_2$.  Asterisks
	denote that a transition is performed under any action.}
\label{fig:ex}
\end{figure}

\begin{example}
\label{ex:wcv-ev}
Figure~\ref{fig:ex} shows a toy POMDP: A mining robot has to mine ore, which 
can be of two types 
(states $t_1$ and $t_2$). The exact type is unknown, but 
$t_1$ is more likely to occur (initial belief $\initd$). The goal is to 
reach the ``ore mined'' 
($\mathit{mnd}$) state, in which a lump-sum reward is received. The robot can 
use several mining modes: safe mode (action $\mathit{ms}$), which succeeds with 
probability $0.6$ and does not do anything if it fails, or type-specific mining
modes ($m_1$ and $m_2$) which succeed if applied on the correct type but result
in a catastrophic failure if used on a wrong type. It can also use a sensor to
accurately determine the type (after which a type-specific action can be safely
used), at a cost of a one-step delay.

An exhaustive analysis of possible policies
reveals that the
expected value is maximized by any policy $\sigma$ which selects $m_1$ in the
first step (we then have $\evalue^\pomdp(\sigma)=45$). However, the worst-case
value of such a policy is $0$, as it can result in entering $\mathit{fail}$
after the first step. On the other hand, a policy $\sigma'$ which plays
$\mathit{sense}$ in the first step has
$\evalue^\pomdp(\sigma)=\wvalue^\pomdp(\sigma)=25$.
\end{example}

\myparagraph{Main Computational Questions.}
The standard \emph{POMDP planning} problem asks to compute (or approximate) the
policy maximizing the expected value.
In \emph{online POMDP planning}, instead
of computing the whole policy we have to compute, in each time step, the best
action in the current situation. In other words, we must compute a good local
approximation of a (near-)optimal policy.  \cite{RPPC08:online-planning-pomdp}.
In contrast, in the \emph{threshold} planning problem we are asked to compute a
policy maximizing the worst-case value and thus provide strict guarantees on the
performance of the system~\cite{ZP:games-graphs}.  In this paper, we combine
these two approaches and study the \emph{guaranteed payoff optimization (GPO)}
problem, where we are given a POMDP $\pomdp$
and a threshold $\thr\in \Rset$ and we have to compute a policy $\sigma$ such that
\begin{compactenum}[a)]
\item $\sigma$ satisfies a \emph{threshold constraint}:
	$\wvalue^\pomdp(\sigma)$ is at least $t$.
\item Let $\gvalue^\pomdp(t)$ denote the best expected value obtainable while
	ensuring a worst-case payoff of at least $t$, i.e.
	$\gvalue^\pomdp(t):=\sup\{\evalue(\pi)\mid 
	\wvalue^\pomdp(\pi)\geq t\}$.
	Among all policies that satisfy item a), $\sigma$ has 
	$\varepsilon$-maximal expected
	value, i.e. 
	\(
		\evalue^\pomdp(\sigma) \ge
		\gvalue^\pomdp(t) - \varepsilon.
	\)
\end{compactenum}

To efficiently tackle the GPO problem we aim to 
compute, in an online fashion, a local approximation of policy $\sigma$ 
above. However, we 
\emph{do not} relax requirement a). Approximations notwithstanding, the 
online planning algorithm we 
seek is such that given $t$, the discounted payoff of every single play that can be 
produced by the algorithm is at least $t$.

\begin{example}
\label{ex:bwc}
Take the POMDP in Figure~\ref{fig:ex} and a threshold $t=5$. As shown in 
Example~\ref{ex:wcv-ev}, a policy $\sigma'$ playing $\mathit{sense}$ in the 
first step satisfies $\wvalue^\pomdp(\sigma')\geq t$. However, there are better 
(w.r.t. the expected value) policies satisfying this constraint. The best such 
policy is a policy $\sigma''$ which twice plays $\mathit{ms}$ and then plays
$\mathit{sense}$. This policy satisfies $\evalue^\pomdp(\sigma'')=37$ and
$\wvalue^\pomdp(\sigma'')=6.25$. (Also note that the optimal policy to maximize
the expected payoff plays $\mathit{m_1}$ at the very start. However, with
non-zero probability, this strategy violates the worst-case threshold $t=5$.)
\end{example}

\section{Policies for GPO Problem}

We first show the GPO problem is different from the classical expectation
maximization.
\begin{example}[Beliefs are not sufficient for GPO.]
It is known that beliefs form a sufficient statistic of history for achieving
the optimal expected value, i.e. there is always a deterministic belief-based
policy $\sigma$ --- that is, a policy such that  for each history $h$ the
distribution $\sigma(h)$ is Dirac and determined solely by the belief after
observing $h$ --- with optimal expected value~\cite{sondik}.  However, beliefs
are not a sufficient statistic for the GPO problem, as witnessed by
Example~\ref{ex:bwc}: suppose that we use policy $\sigma''$ and consider
histories $h=\mathit{ms}~o~\mathit{ms}~o$ and $\bar{h}=\mathit{ms}~o$, where
$o$ is the observation received in $t_1$ and $t_2$. The beliefs $b_{h}$ and
$b_{\bar{h}}$ are identical, and yet $\sigma''(h)\neq \sigma''(\bar{h})$, i.e.
$\sigma''$ is not belief-based.
\end{example}

\myparagraph{Overview of Policy Representation.}
We show (in Corollary~\ref{cor:sufficient-statistic}) that a sufficient 
statistic 
for solving the GPO 
problem is a 
tuple $(\belief_h,\remain{h}{\discount}{t})$, where $\belief_h$ is the belief 
after history 
$h$ and $\remain{h}{\discount}{t}$ is the ``remaining'' distance to the 
threshold which we need to accumulate in the future. Formally, 
\[
	\remain{h}{\discount}{t} =
	\left(t-\min \{\discpath{w}{\discount}\mid
	\histfunc(w)=h\}\right)/\gamma^{\len{h}}.
\]
This is similar to other 
(PO)MDP planning problems that work with 
thresholds~\cite{White:threshold-mdps,HYV16:risk-pomdps}. However, we prove 
more: we obtain a precise local characterization of policies that satisfy 
the threshold constraint. More precisely, we show that 
for each history $h$, there is a 
set of
\emph{allowed} actions $\allowed{t}{\discount}(h)$ such that a policy 
$\sigma$ 
satisfies 
$\wvalue^P(\sigma)\geq t$ if and only if for each history $h$ it holds 
$\supp(\sigma(h))\subseteq \allowed{t}{\discount}(h)$. We show that the 
function $\allowed{t}{\discount}$ can be finitely 
represented 
and, for any history $h$, its value can be
computed algorithmically. This permits us to split the solution of the 
GPO problem into two separate parts: 1.) We compute
the function $\allowed{t}{\discount}$, and 2.) we use it to restrict 
a  
standard
online planning algorithm so that it always returns an action 
allowed for the current history.

\myparagraph{Allowed Actions $\allowed{t}{\discount}$.}
Intuitively, an action 
$a$ should be allowed after some history $h$ only if the payoff we are guaranteed 
to accumulate using $a$ in the current step (i.e. $\min_{s\in \supp(\belief_h)} 
r(s,a)$) plus the best payoff which we 
can guarantee from the next step onward is at least $\remain{h}{\discount}{t}$. 
To formalize the ``best payoff guaranteed from the next step on'' we define the 
\emph{future value} of any history $h$ as
\[ \textstyle
\fvalue(h) = \sup_{\sigma} \wvalue^{\pomdp[\belief_{h}]}(\sigma),
\]
where $\pomdp[\belief_{h}]$ is a POMDP identical to $\pomdp$ except for having 
initial belief $\belief_{h}$ and the supremum is taken over all policies in 
$\pomdp[\belief_{h}]$.

\myparagraph{Belief Supports Suffice for the Worst Case.}
The crucial observation is that the future value of a history $h$ is determined 
only by the support of $b_h$.
\begin{lemma}\label{lem:only-belief-support}
If histories $h,h'$ in a POMDP $\pomdp$ are such that 
$\supp{(b_{h})}=\supp{(b_{h'})}$, 
then $\fvalue(h)=\fvalue(h')$.
\end{lemma}

Intuitively, this is because the worst-case value of a policy (and thus also a 
future value of a history) does not 
depend on any transition probabilities.
In a slight abuse of notation,
we sometimes treat $\fvalue$ as a function from $2^\states$ to $\Rset$, i.e. 
$\fvalue(B)$, for $B\subseteq S$, is equal to $\fvalue(h)$ for all 
histories $h$ such that $\supp(b_h)=B$. 

\myparagraph{$\valbound$ as an Approximation of $\fvalue$.}
Since computing $\fvalue(B)$ exactly can be inefficient in practice, we often 
need to work with approximations of $\fvalue(B)$, without relaxing the 
threshold 
constraint. We thus introduce a notion of a 
$\valbound$-allowed 
action.
Let $\valbound\colon 2^{\states} \rightarrow \Rset$ be a function 
assigning numbers to belief supports. We say that an action $a$ is 
\emph{$\valbound$-allowed} for $t\in \mathbb{R}$ after history $h$, and
write it 
$a\in \Psi\mhyp\allowed{t}{\discount}(h)$, if for all
states $s\in \supp(b_h)$ and all observations $o\in\obs$ such that $hao$ is a
history it holds that
\begin{equation}
\label{eq:allowed-act}
r(s,a)+\discount\cdot\valbound(\supp(b_{hao})) \geq 
\remain{h}{\discount}{t}.
\end{equation}
If $\Psi$ is the function $\fvalue$, we write simply 
$a\in \allowed{t}{\discount}(h)$. 
We typically aim at 
computing a lower bound on $\fvalue$, i.e. a function $\valbound$ such that 
$\valbound(\belsup)\leq \fvalue(\belsup)$ for each $\belsup \in 2^\states$. 
Then, as shown below, playing $\valbound$-allowed actions still guarantees that 
the threshold $t$ is 
eventually surpassed. 

\myparagraph{Correctness of the Approximation.}
The correctness of the definition is summarized in the following proposition. 
We 
say that a policy 
$\sigma$ is $\valbound$-safe for $t\in \Rset$ if for each history 
$h$ consistent with $\sigma $ 
it holds that $\supp(\sigma(h)) 
\subseteq \Psi\mhyp\allowed{t}{\discount}(h)$. 

\begin{proposition}
\label{prop:allowed}
	Let $\valbound\colon 2^\states \rightarrow \Rset$ be a function such that
	$\valbound(B)\leq \fvalue(B)$ for each $B\in 2^\states$.  Then any
	policy $\sigma$ that is $\Psi$-safe for $t$ satisfies
	$\wvalue^{\pomdp}(\sigma)\geq t$. Moreover a policy $\pi$ is
	$\fvalue$-safe for $t$ if and only if $\wvalue^{\pomdp}(\pi)\geq t$.
\end{proposition}

\begin{corollary}
\label{cor:sufficient-statistic}
	Assume that there is a policy $\sigma$ with $\wvalue^P(\sigma)\geq
	t$.  Then there is also a policy $\pi$ such that $\wvalue^P(\pi) \geq t$
	and $\evalue^P(\pi)=\gvalue^\pomdp(t)$, and moreover, $\pi$ is
	belief-and-payoff, based, i.e. for all histories $h,h'$ such that
	$(b_h,\remain{h}{\discount}{t})=(b_{h'},\remain{h'}{\discount}{t})$ it
	holds $\pi(h)=\pi(h')$.
\end{corollary}

From~\eqref{eq:allowed-act} we see that to compute $ 
\allowed{t}{\discount}(h)$ we have to keep track of $\remain{h}{\discount}{t}$ 
(which 
can be easily done online) and to compute $\fvalue(\supp(b_h))$ (or a 
suitable under-approximation thereof). In the next section we show how to do 
the latter. 

\begin{example}\label{ex:allowed}
	Consider the POMDP from Figure~\ref{fig:ex} with a threshold $t = 12$. 
	Then $\fvalue(\{fin\}) =
	\fvalue(\{\mathit{fail}\}) = 0$,
	$\fvalue(\{t_1,t_2\}) = 25$, $\fvalue(\{t'_1\}) = \fvalue(\{t'_2\}) =
	50$, and $\fvalue(\{\mathit{mnd}\}) = 100$. Initially, for the empty 
	history, we have
	$\remain{\cdot}{0.5}{12} = 12$ and therefore the only allowed actions are
	$\mathit{ms}$ and $\mathit{sense}$ because for all $i \in \{1,2\}$ we have
	\(
		r(t_i,m_{3-i}) + \discount \fvalue(\{\mathit{fail}\}) = 0 <
		\remain{\cdot}{0.5}{12} = 12.
	\)
	Suppose that $\mathit{ms}$ is played and that the next observation
	witnessed is $\obsfunc(t_1) = \obsfunc(t_2)$ (thus, the belief is the same as
	before). We have $\remain{\mathit{ms}\obsfunc(t_1)}{0.5}{12.5} = 25$.
	In this case, the only allowed action is $\mathit{sense}$ because for
	all $i \in \{1,2\}$
	\(
		r(t_i,\mathit{ms}) + \discount \fvalue(\{t_1,t_2\}) = 12.5 <
		\remain{\mathit{ms}\obsfunc(t_1)}{0.5}{12} = 24
	\)	
	and $m_1$ and $m_2$ are still not allowed (since we have not accumulated
	any payoff and have the same belief as before). Hence, $\mathit{sense}$
	is played and consequently we obtain a payoff of $12.5$ (because of 
	discounting). We remark that $12.5$
	is, as required, above the threshold $t = 12$.
\end{example}

\section{Computing Future Values}\label{sec:future-values}

The threshold constraint in the GPO problem is \emph{global},
i.e. it talks about \emph{all} runs 
compatible with a policy. Hence, solving the GPO problem is 
unlikely to be amenable to \emph{purely} online methods, which compute only 
local 
approximations of policies. In this 
section we show how to compute future values in an offline pre-processing step.
Although this requires a global analysis of a POMDP, the pre-processing step can
be done efficiently since computation of future values only requires working
with belief supports rather than beliefs.

\myparagraph{Belief Supports \& Valid Belief Supports $\ValBelSup$.}
A belief support $B\subseteq 2^\states$ is \emph{valid} if either 
$B=\supp(\initd)$ or there is a history $h$ such that $B=\supp(b_h)$. Only 
valid supports can be encountered during the planning process and thus 
we only need to compute future values thereof. We denote by 
$\ValBelSup(\pomdp)$ the set of valid belief supports of POMDP $\pomdp$; the 
set can be computed by a simple iterative procedure. 

\myparagraph{Obsevable Rewards.}
We present efficient computation of future values under the assumption that 
\emph{rewards 
are observable}. This holds for many real-world applications, see, e.g. 
examples in~\cite{HYV16:risk-pomdps,CCGK15}. Formally, POMDP $\pomdp$ has 
observable rewards if $\reward(s,a)=\reward(s',a)$ whenever 
$\obsfunc(s)=\obsfunc(s')$. From a theoretical point of view, observability of 
rewards is necessary since without it, the computation of future values is at
least as hard as solving a long-standing open problem in algebraic number
theory.
More precisely, if the rewards of a given POMDP are not observable, the
computation of future values is at least as hard as solving the \emph{target
discounted sum problem}, a long-standing open problem in automata theory related
to other open problems in algebra~\cite{BHO15:target-disc-sum}.
However, for POMDPs with
unobservable rewards we can at least obtain an under-approximation $\valbound$ of
$\fvalue$, and hence our framework is also applicable to them.

\begin{lemma}
\label{lem:obsrew}
If rewards in $\pomdp$ are observable, then for each $B\in\ValBelSup(\pomdp)$ 
and each $s,s'\in B,a\in \act$ it holds $\reward(s,a)=\reward(s',a)$.
\end{lemma}
We thus define $\reward(B,a)$ as $\reward(s,a)$ for some $s\in B$.

\myparagraph{Future Value Characterization.}
We start by providing a
characterization of future values. A 
\emph{successor} of a belief support $B$ under action $a$ and observation $o$ 
is a belief support 
$\succb(B,a,o)=o \cap \bigcup_{s\in B} \supp(\delta(s,a))$. Consider the 
following system of $\max$-$\min$ equations with variables $x_B$, 
$B\in 
\ValBelSup(\pomdp)$:
\begin{equation}
\label{eq:bellman}
x_B = \max_{a\in\act}\min_{\substack{o\in \obs\\ \succb(B,a,o)\neq \emptyset}} 
\reward(B,a) + \discount\cdot
x_{\succb(B,a,o)}.
\end{equation}
(Each $B\in\ValBelSup(\pomdp)$ appears on the LHS of exactly one equation in the 
system.)

\begin{proposition}
\label{prop:fval-char}
The system~\eqref{eq:bellman} has a unique solution 
$\{\tilde{x}_B\}_{B\in 
\ValBelSup(P)}$, and it satisfies $\tilde{x}_B = \fvalue(B)$.
\end{proposition}

\myparagraph{Game Perspective for the Worst Case.}
Hence, it suffices to find a solution to system~\eqref{eq:bellman}. But the 
form of the system is identical to the one characterizing optimal values 
in 
2-player zero-sum discounted games~\cite{ZP:games-graphs}. 
These games can be 
imagined as 
fully-observable MDPs in which the outcomes of actions are not resolved by 
a random choice but by a malicious adversary. The system~\eqref{eq:bellman}
\textit{per se} corresponds to a game where elements of $\ValBelSup(\pomdp)$ 
are the states, 
actions are the same as in $\pomdp$, and possible effects of actions are 
given by the function~$\succb$. 

\myparagraph{Algorithms to Compute Future Values.}
Hence, to compute future values in practice we can employ one of several efficient 
algorithms 
for 
solving discounted-sum games (e.g.~\cite{romain}). 
A simple yet efficient 
approach is to use the standard \emph{value iteration} for games: we compute a 
sequence $f^{(0)} f^{(1)} f^{(2)} \dots$ of functions of type 
$\ValBelSup(\pomdp)\rightarrow \Rset$ such that $f^{(0)}(B)=0$ for each $B$, and
for $i\geq 1$ we inductively define
\[
	f^{(i)}(B) = \max_{a\in\act}\min_{\substack{o\in \obs\\
	\succb(B,a,o)\neq \emptyset}} \reward(B,a) + \discount\cdot
	f^{(i-1)}(\succb(B,a,o)).
\]
From~\cite{ZP:games-graphs} it follows 
there is always $j$ such that for all $B \in \ValBelSup(\pomdp)$ we have
$f^{j}(B)=f^{j-1}(B)$,
i.e. $f^{j}(B)$ is the solution 
to~\eqref{eq:bellman}, and moreover $j\leq 
3+\log_2(\max_{(s,a)\in\states\times\act}|r(s,a)|)+\frac{1}{2}\cdot(|\states|+3)^2
\cdot\frac{
\log_2(\mathit{den}(\discount))}{1-\discount}$, where $\mathit{den}(\discount)$ 
is a 
denominator of $\discount$ in its reduced form. Hence, the value iteration 
converges in at most exponentially many steps.\footnote{Since the number 
$\frac{1}{1-\discount}$ can be exponential in the bitsize of $\discount$.}

\begin{theorem}
\label{thm:fval-complexity}
Future values of all valid belief supports in $\pomdp$ can be computed in time 
exponential in the size of $\pomdp$.
\end{theorem} 

Although the theoretical bound
is exponential, there are several 
reasons for the method to work well in practice:
(1.) In a concrete instance, the number of 
valid supports can be significantly smaller than exponential.
(2.) Reaching the fixed-point of the value iteration may also require 
significantly 
smaller number of steps than the theoretical upper bound suggests.
(3.) One can show that for each $i\geq 0$, $f^{(i)}\leq 
\fvalue$. Hence, even if reaching the fixed point takes too much time, we 
can set up a suitable timeout after which the value iteration is stopped, say at 
iteration $i$. Then, by Proposition~\ref{prop:allowed} any policy that is 
$f^{(i)}$-safe for $t$ has worst-case value $\geq t$.
(4.) Value iteration is a simple and standard algorithm for which
efficient implementations exist (see, e.g.,~\cite{ldk95,sv05}).

\myparagraph{Important note on $\valbound$:} generally, $\valbound\leq 
\fvalue$ does not guarantee that a $\valbound$-safe policy exists, which is 
necessary to apply Proposition~\ref{prop:allowed}. The 
following lemma resolves this.

\begin{lemma}
For any $i\geq 0$ the following holds for the functions $f^{(i)}$ produced by 
game value iteration: if $f^{(i)}(\supp(\initd))\geq t$, then there exists a 
policy $\sigma$ which is $f^{(i)}$-safe for $t$.
\end{lemma} 

In particular, if $\fvalue(\supp(\initd))\geq t$ then a 
$\fvalue$-safe policy for $t$ exists, irrespective of the way in 
which $\fvalue$ is computed.

\begin{figure*}
	\centering
		\includegraphics{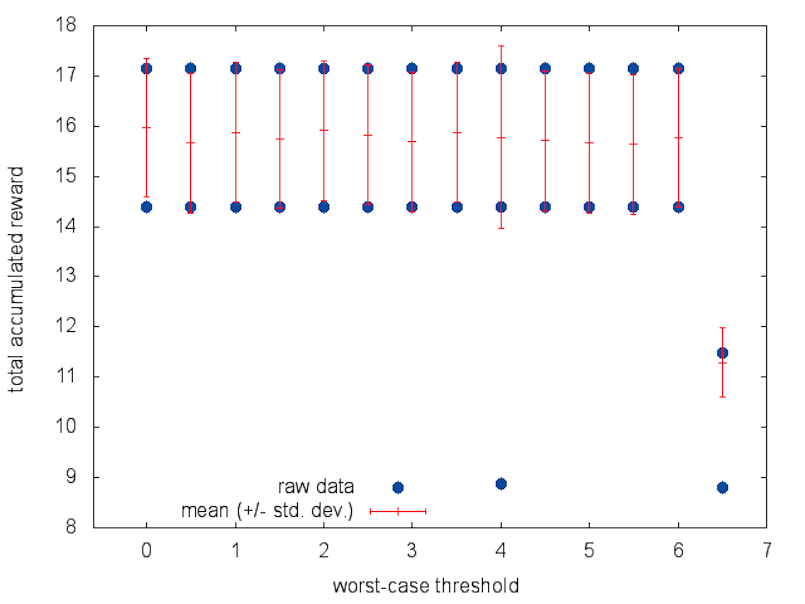}
		\includegraphics{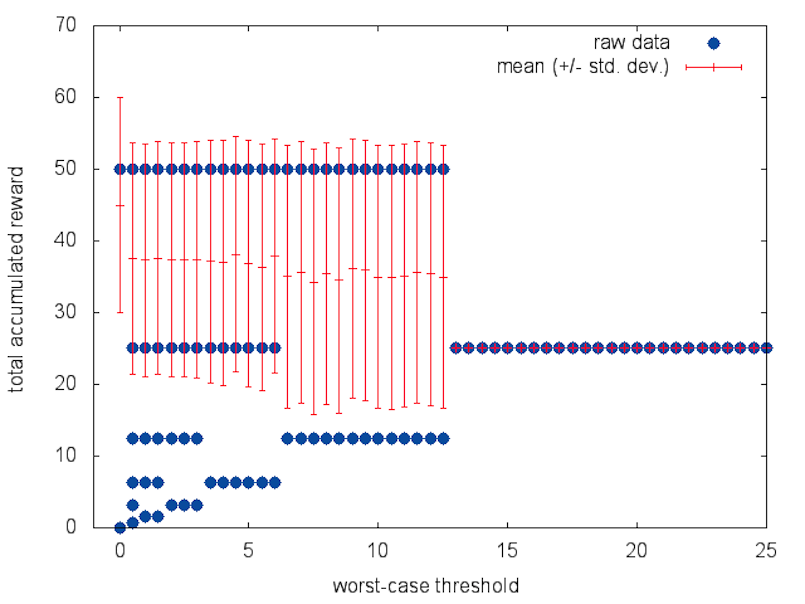}
		\includegraphics{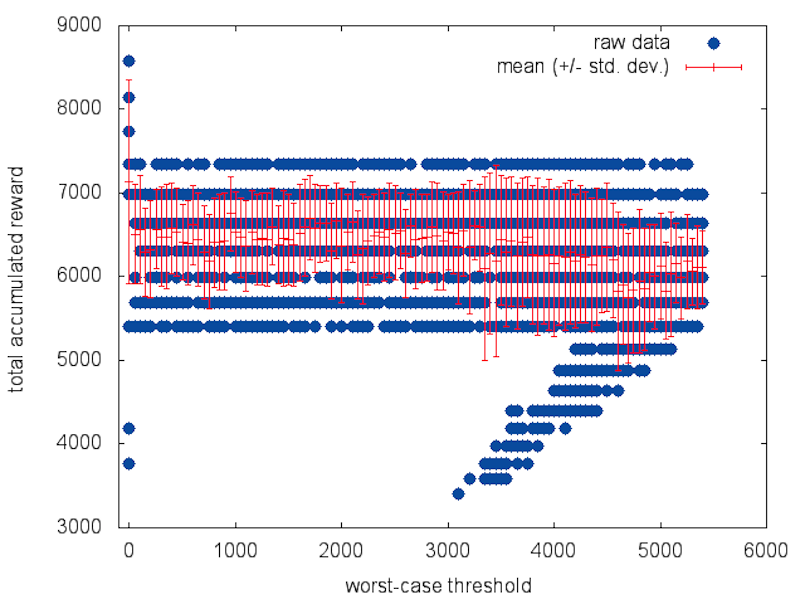}
	\caption{Plots of results obtained from simulating (1.) a RockSample
		benchmark, (2.) the POMDP from Example~\ref{ex:wcv-ev}, and (3.)
		a hallway benchmark with probabilistic spinning (a.k.a. traps),
		all with increasing worst-case thresholds (until
		$\fvalue(\supp(\initd))$). Each circle with
		coordinates $(x,y)$ corresponds to a simulation of G-POMCP, ran
		with worst-case threshold $x$, that obtained $y$ as accumulated
		payoff. The vertical bars show the mean and standard
		deviation per worst-case threshold. (We have plotted at least
		100 data-points per worst-case threshold for the RockSample
		benchmark; 1000 for Example~\ref{ex:wcv-ev}; 20 for the hallway
		benchmark.)
	}
	\label{fig:plots-benchmarks}
\end{figure*}

\section{Solving the GPO problem}

We solve the GPO problem by modifying the partially-observable Monte Carlo 
planning (POMCP) algorithm~\cite{SV:POMCP}.

\myparagraph{POMCP.}
POMCP is an online planning method 
which in each decision epoch aims to select the best action given the current 
history $h$.  
In each epoch,
POMCP performs a number of 
finite-horizon simulations starting from belief $b_h$ in order to 
compute a local approximation of the optimal expected value function: each 
simulation extends history $h$ by selecting actions according to 
certain rules 
until the horizon is reached. The payoff of the produced path is then 
evaluated, and the result is used to update the optimal value approximation. 
After all 
the simulations proceed, the best action according to the estimated values is 
played, a new observation is received, and the process continues as above.

\myparagraph{POMCP data-structure.}
POMCP 
stores the information gained in past simulations in a \emph{search tree}, 
in which
each node corresponds to some history $h'$ and contains belief 
$b_{h'}$, the 
number $N_{h'}$ of times the history has been observed in previous simulations, 
and 
an approximation of the optimal expected value from $b_{h'}$.
The search tree is used to guide simulations: each step in which the 
current history corresponds to an internal node of the tree is 
treated as a multi-armed bandit with parameters determined by numbers stored in 
children of this 
node, 
which balances exploration of new branches and 
exploitation of previous simulations (akin to the UCT algorithm for 
MDPs~\cite{KS06}). Once the simulation runs out of the 
scope of the search tree, it enters a \emph{rollout} phase, where a fixed 
policy (e.g. selecting actions at 
random) is used to extend paths. 

\myparagraph{G-POMCP: Adapting POMCP for GPO.}
We propose an augmentation of POMCP, which we call G-POMCP (\emph{guaranteed} 
POMCP), specified as follows: First we enrich the nodes of the search tree so
that
a node corresponding to  
a history $h$ additionally includes the set $B_h= \supp(b_h)$ and the number 
$R_h = \remain{h}{t}{\discount}$. When adding a new node to a search tree by 
extending history $h$ with action $a$ and observation $o$, these attributes for 
the new node are updated 
as follows:
$B_{hao}=\succb(B_h,a,o)$ and 
$R_{hao}=(\remain{h}{t}{\discount} 
- \reward(B_h,a))/\discount$. Note that updating $B_h$ to $B_{hao}$ requires 
just discrete set operations; as a matter of fact, the function $\succb$ is 
computed already during the off-line computation of future values, after which 
it can be stored and used to efficiently update $B_h$ during G-POMCP execution. 
In particular, updating $B_h$ is independent 
of updating $b_h$, which is important so as not to compromise the threshold 
constraints with issues of belief precision and particle 
deprivation. 

\myparagraph{G-POMCP: playing safe.}
The execution of G-POMCP then proceeds in almost the same way as 
in POMCP, with a crucial exception: Whenever G-POMCP is to select a (real or 
simulated) action 
it selects only among those in $\allowed{t}{\discount}(h)$, where $h$ is the 
current history. 
Note 
that checking whether an action is 
allowed is easy for histories within the search tree, since the necessary 
information ($B_{h}$ and $R_{h}$) is 
stored 
in 
nodes of the tree. Out of the scope of the search tree, we need to update the 
current belief support and remaining payoff online, as the simulation proceeds. 
While this somewhat increases the complexity of rollouts, as current belief 
supports must be kept updated (POMCP only keeps track of the current 
state and 
of payoff 
won so far), as noted above, updating belief supports is easier than 
updating beliefs. Moreover, this increase in complexity is only an issue in the 
initial steps of the algorithm, where rollout steps dominate over tree 
traversal. 
Previous sections yield the 
following result:

\begin{theorem}
\label{thm:G-POMCP-corr}
For each threshold $t\leq \fvalue(\supp(\initd))$ the following holds: for each
play $\rho = s_0 a_0 s_1 a_1\dots$ resulting from using G-POMCP on $\pomdp$
\textit{ad infinitum} it holds $\discpath{\rho}{\discount}\geq t$. This holds
independently of how precisely the algorithm approximates beliefs.
\end{theorem}
So unless it is impossible to satisfy the threshold constraint at all, it can 
be surely satisfied by using G-POMCP.

\myparagraph{Convergence.}
Another question is the one of convergence. An algorithm is said to be 
convergent in the limit if, assuming precise belief representation, the local 
approximation of optimal value converges to true optimal value (in our case 
to $\gvalue^\pomdp(t)$) as the number 
of simulations and their depth increases. 
The limit convergence of G-POMCP can be proved by a straightforward adaptation 
of the limit convergence proof of POMCP~\cite{SV:POMCP}: 
we map executions of 
G-POMCP on POMDP $\pomdp$ to the executions of UCT on a tree-shaped 
MDP $\pomdp'$, whose states are histories of $\pomdp$ (with the empty history as
root) and where finite paths correspond to extending histories in $\pomdp$ by 
playing allowed actions.

\section{Experiments}
We tested our algorithm on two classical sets of benchmarks. The first, Hallway,
was introduced in~\cite{lck95}. In a hallway POMDP, a robot navigates a
gridworld with walls and traps. We have considered variants in which traps cause
non-recoverable damage and another in which they just ``spin'' the robot ---
making him more uncertain about his current location in the grid. Additionally,
we have run our algorithm on RockSample POMDPs. The latter corresponds to the
classical scenario described first in~\cite{ss04}. (We use a slight adaptation 
with a single imprecise sensing action.) Our 
experimental results are summarized in Figure~\ref{fig:plots-benchmarks} and
Table~\ref{tab:latency}.

\myparagraph{Test Environment Specifications:}
\begin{inparaenum}[$(1.)$]
	\item CPU: $6$-Core Intel Zeon, $3.33$ GHz, $6$ cores;
	\item Memory: $256$ KB of L$2$ Cache, $12$ MB of L$3$ Cache, $32$
		GB;
	\item OS: Mac OS X $10.7.5$.
\end{inparaenum}

\myparagraph{Worst-Case vs. Expected Payoff.} In
Figure~\ref{fig:plots-benchmarks} we have plotted the results of running our
G-POMCP algorithm on several benchmarks. In all three graphics, the
trade-off between worst-case guarantees and expected payoff is clearly visible:
In the left figure, the expected payoff stays around $15.7$ for worst-case
thresholds between $0$ and $6$; then drops to $11.3$ for threshold values above
$6.5$. In the center figure, the expected payoff is ${\sim}44.7$ when the worst-case
threshold is $0$; stays around $36$ for thresholds between $1$ and $12$ (with a
slightly negative slope); then drops to $25$ for threshold values above
$12.5$. Finally, in the right figure, the expected payoff steadily decreases for
increasing worst-case threshold values. In particular, for threshold $0$ the
expected payoff is ${\sim}7137$ while for threshold $5150$ it is ${\sim}6161$.

\myparagraph{Latency.} In Table~\ref{tab:latency} we show the
\emph{latency} --- the amount of time it takes to determine, at each epoch,
which action to play next --- of G-POMCP on three of the benchmarks we
considered. (Though we have run the tool on several others, these are
the biggest.) Observe that, even for relatively big POMDPs, the average latency
is in the order of seconds. Also, note that the pre-processing step is not too
costly.

\begin{table}
	\centering
	\begin{tabular}{ l | c | c | c | c | c }
		No.
		& states
		& act.
		& obs.
		& pre. proc.
		& avg. lat.\\
		\hline
		\hline
		tiger
		& $7$
		& $4$
		& $6$
		& $< 0.001$s
		& $< 0.009$s\\
		\hline
		r.sample
		& $10207$
		& $7$
		& $168$
		& $184$s
		& $0.816$s\\
		\hline
		hallway
		& $2039$
		& $3$
		& $18$
		& $2.02$s
		& $1.308$s\\
	\end{tabular}

	\caption{Latency of G-POMCP with planning horizon of $1$K}
	\label{tab:latency}
\end{table}

\myparagraph{Tool Availability.} Our implementation of the G-POMCP algorithm can
be fetched from \texttt{https://github.com/gaperez64/GPOMCP}.

\section{Discussion}
In this work we have given a practical solution for the GPO problem. Our
algorithm, G-POMCP, allows to obtain a policy which ensures a worst-case
discounted-sum payoff value while optimizing the expected payoff. We have
implemented G-POMCP and evaluated its performance on classical families of
benchmarks. Our experiments show that our approach is efficient despite the
exact GPO problem being fundamentally more complicated.

\section*{Acknowledgements}
The research leading to these results was supported by the Austrian
Science Fund (FWF) NFN Grant no. S11407-N23 (RiSE/SHiNE); two ERC Starting grants
(279307: Graph Games, 279499: inVEST); the Vienna Science and Technology Fund (WWTF)
through project ICT15-003; and the People Programme (Marie Curie Actions) of
the European Union's Seventh Framework Programme (FP7/2007-2013) under
REA grant agreement no. [291734].

{\small
\bibliographystyle{alpha}
\bibliography{bibliography-master,new}
}

\clearpage
\appendix
\begin{center}
{\Large Technical Appendix}
\end{center}

\section{Examples of Section~\ref{sec:prelims}}
Here is presented a detailed analysis of all possible policies, and the best 
policy in terms of optimized expected payoff. Firstly observe that a policy is 
uniquely determined if the first performed action is in the set 
$\{\mathit{m_1},\mathit{m_2},\mathit{sense}\}$. The remaining case is to 
perform action $\mathit{ms}$ $n$ times for some $n\in \mathbb{N}$ (if we 
successfully make transition to $\mathit{mnd}$ before performing all $n$ 
actions $\mathit{ms}$, policy is still uniquely determined), and then perform 
some action in the set $\{\mathit{m_1},\mathit{m_2},\mathit{sense}\}$. 
Alternatively, it is possible to just perform $\mathit{ms}$ until 
$\mathit{mnd}$ is successfully reached. Below are computed expected payoffs for 
each of the cases listed above.

\begin{itemize}
	
	\item $\sigma_1$: $\mathit{m_1}$ performed first
	\begin{flalign*}
	\evalue^\pomdp(\sigma_1) &= 0.9\cdot[1\cdot 0+\gamma\cdot 100]=45.&&
	\end{flalign*}
	
	\item $\sigma_2$: $\mathit{m_2}$ performed first
	\begin{flalign*}
	\evalue^\pomdp(\sigma_2) &= 0.1\cdot[1\cdot 0+\gamma\cdot 100]=5.&&
	\end{flalign*}
	
	\item $\sigma_{\mathit{sense}}$: $\mathit{sense}$ performed first
	\begin{flalign*}
	\evalue^\pomdp(\sigma_{\mathit{sense}}) &= 0+\gamma\cdot 0+\gamma^2\cdot 100=25.&&
	\end{flalign*}
	
	\item $\sigma_{\mathit{ms}}$: $\mathit{ms}$ performed until transition to 
	$\mathit{mnd}$ is successful
	\begin{flalign*}
	\evalue^\pomdp(\sigma_{\mathit{ms}}) &= \sum_{k=0}^{\infty}(\frac{2}{5})^k\cdot 
	\frac{3}{5}\cdot \gamma^{k+1}\cdot 100=37.5.&&
	\end{flalign*}
	
	\item $\sigma^n_1$: $\mathit{ms}$ performed $n$ times, then $m_1$
	\begin{flalign*}
	\evalue^\pomdp(\sigma^n_1) &= \sum_{k=0}^{n-1}(\frac{2}{5})^k\cdot 
	\frac{3}{5}\cdot \gamma^{k+1}\cdot 100&&\\
	&+ (\frac{2}{5})^n\cdot [0.9\cdot \gamma^{n+1}\cdot 100+ 0.1\cdot 
	0]=37.5+\frac{7.5}{5^n}.&&
	\end{flalign*}
	
	\item $\sigma^n_2$: $\mathit{ms}$ performed $n$ times, then $m_2$
	\begin{flalign*}
	\evalue^\pomdp(\sigma^n_2) &= \sum_{k=0}^{n-1}(\frac{2}{5})^k\cdot 
	\frac{3}{5}\cdot \gamma^{k+1}\cdot 100&&\\
	&+ (\frac{2}{5})^n\cdot [0.1\cdot \gamma^{n+1}\cdot 100+ 0.9\cdot 
	0]=37.5-\frac{32.5}{5^n}.&&
	\end{flalign*}
	
	\item $\sigma^n_{\mathit{sense}}$: $\mathit{ms}$ performed $n$ times, then 
	$m_{\mathit{sense}}$
	\begin{flalign*}
	\evalue^\pomdp(\sigma^n_{\mathit{sense}}) &= \sum_{k=0}^{n-1}(\frac{2}{5})^k\cdot 
	\frac{3}{5}\cdot \gamma^{k+1}\cdot 100&&\\
	&+ (\frac{2}{5})^n\cdot \gamma^{n+2}\cdot 100=37.5-\frac{12.5}{5^n}.&&
	\end{flalign*}
	
\end{itemize}

It is hence clear that in Example 1 the expected payoff is optimized for 
$\sigma=\sigma_1$. In Example 2 though, if we introduce a threshold $t=5$, this 
policy does not work as if the initial state is $t_2$, payoff is $0$. Looking 
above at possible policies, $\sigma_1$, $\sigma_2$, $\sigma^n_1$ and 
$\sigma^n_2$ do not satisfy the imposed worst-case condition as we may have 
payoff $0$. If $\mathit{ms}$ returns us to the initial state for at least three 
times, total payoff is at most $100/2^5=3.125<5$, so $\sigma_{\mathit{ms}}$ and 
$\sigma^n_{\mathit{sense}}$ also do not satisfy the condition for $n\geq 3$. 
Hence, policies satisfying the worst case condition are 
$\sigma_{\mathit{sense}}$ and $\sigma^n_{\mathit{sense}}$ for $n\in \{1,2\}$. 
It is easily verified from above that $\sigma''=\sigma^2_{\mathit{sense}}$ 
optimizes expected payoff with $\evalue^\pomdp(\sigma^2_{\mathit{sense}})=37$, and the 
worst case is achieved if both $\mathit{ms}$ fail with 
$\wvalue^\pomdp(\sigma^2_{\mathit{sense}})=6.25$.

\section{On the assumption of observable rewards
	(Section~\ref{sec:future-values})}\label{sec:assumption-observable}
If the rewards of a given POMDP are not observable,
the computation of future values is at least as hard as solving the
\emph{target discounted sum problem}, a long-standing open problem in automata
theory related to other open problems in algebra~\citetrsec{BHO15:target-disc-sum}.

\paragraph*{Under-approximation of $\fvalue$.}
For POMDPs with non-observable rewards, there is a straightforward way of
obtaining an under-approximation $\valbound$ of $\fvalue$. Following the value
iteration algorithm for discounted-sum games outlined in
Section~\ref{sec:future-values} and detailed in~\citetrsec{HM15}, it is possible to
obtain the exact future values. Furthermore, it is easy to see that the
functions $f^{(i)}$ generated by the algorithm get ever closer to the actual
future values. Hence, stopping the iteration at any $i \ge 0$ yields the desired
under-approximation. (Note that for this argument to be valid, the reward
function must assign to every transition a non-negative value. However, this
assumption is no loss of generality since, for any given POMDP, the threshold and the rewards of all
the transitions can be ``shifted and scaled'' so that the assumption holds.)

\section{Formal Proof of Lemma~\ref{lem:only-belief-support} and
Theorem~\ref{thm:fval-complexity}}
In this section we argue that, for POMDPs with observable rewards, we can reduce the
computation of a policy with worst-case value above a given threshold to the
computation of a policy, with the same property, in a \emph{full-observation
discounted-sum game}. This will give us access to the theoretical tools
developed for that kind of game by the formal verification community. The idea
is simple: we will construct a \emph{weighted arena} in which states
correspond to subsets of states from the POMDP with the same observation, and
the new transitions model transitions with non-zero probability in the POMDP. This
\emph{subset construction} captures the fact that in a POMDP, after any history,
any one from a set of possible states with the same observation could be the
actual state of the system. The assumption that the POMDP has observable rewards
will then allow us to weight the transitions of the arena without losing information
about the original POMDP. 

We observe that this reduction, and the fact that the policy we are looking
for in the original POMDP can be directly obtained from the constructed
discounted-sum game, imply that the probabilities of the POMDP do not really
matter when considering the worst-case value. Thus,
Lemma~\ref{lem:only-belief-support} follows.

Given a POMDP $\game=(\states,\act,\trans,\reward,\obs,\obsmap,\initd)$ with
observable rewards, we construct the weighted arena $\Gamma_\game =
(Q,I,\act,\Delta,w)$ where:
\begin{itemize}
	\item~$Q = \{T \subseteq \states \mid T \neq \emptyset \text{ and }
		\obsmap(s) = \obsmap(s') \text{ for all } s,s' \in T \}$ is a
		finite set of states;
	\item~$I = \{q \in Q \mid \supp(\initd) \cap q = q\}$ is the
		set of initial states;
	\item~$\Delta \subseteq Q \times \act \times Q$ includes transitions of the
		form $(q,a,q')$ if $q,q' \in Q$ and $\bigcup_{s \in
		q} \supp(\trans(s,a)) \cap \obsmap(s') = q'$ for any $s' \in
		q'$;
	\item~$w$ is a weight function of the form $\Delta \to \mathbb{R}$
		determined by $\reward$ as follows: $w(q,a,q') = r(s,a)$
		for any $s \in q$.
\end{itemize}
A play or infinite path in a weighted arena is a sequence $\pi = q_0 a_0
\ldots)$ of states and actions s.t. $q_0 \in I$ and for all $i \geq 0$ we have
$q_i a_i q_{i+1} \in \Delta$. We denote by $\Pi$ the set of all plays. A
(finite) path is a finite prefix of a play ending in a state. Since the game has
full observation, a history in a weighted arena is simply a path. The discounted
sum of a play is defined as for POMDPs but using the weight function $w$ instead
of $\reward$. The definitions for policy and worst-case value are then
identical. (For clarity, we write $\wvalue'$ instead of $\wvalue$ when referring to
the worst-case value in $\Gamma_\game$.)

\paragraph*{From histories of the POMDP to histories in the game.} We now
define a mapping $\mu$ from observation-action sequences to state-action
sequences in the constructed weighted arena. For a history $h = a_0 o_0
\ldots$ from $\game$ we let $\mu(h) = q_0 a_0 \ldots$ where $q_0 = 
\supp(\initd)$ and for all $i \ge 0$ we have $q_{i+1} = \bigcup_{s \in q_i}
\supp(\trans(s,a)) \cap o_{i+1}$. 

\begin{claim}
	The function $\mu$ is a bijective function from histories in $\game$ to
	paths in $\Gamma_\game$.
\end{claim}
\begin{proof}
	Clearly $\mu$ is injective. We will argue that it is also bijective.
	Consider a path $\pi = q_0 a_0 \ldots q_n$ from $\Gamma_\game$. We have
	that $\mu^{-1}(\pi) = a_0 \ldots o_n$ where $o_i = \obsmap(s_i)$ for
	any $s_i \in q_i$ and for all $0 < i \le n$. It remains to show that
	there is a path $\rho = s_0 a_0 \ldots s_n$ in $\game$ s.t.
	$\histfunc(\rho) = \mu^{-1}(\pi)$, to conclude that $\mu^{-1}(\pi)$ is a
	valid history in $\game$. By construction of $\Gamma_\game$ we have
	that, for all $0 < i\le n$, for all states $s' \in q_i$ there is $s \in
	q_{i-1}$ s.t. $\trans(s'|s,a) > 0$. The result follows by
	induction.
\end{proof}
It follows that there are bijective mappings from policies in $\game$ to
policies in $\Gamma_\game$, and from plays in $\game$ to plays in
$\Gamma_\game$. For a policy $\sigma$ in $\game$, let us denote by
$\mu(\sigma)$ the corresponding policy in $\Gamma_\game$; for a play $\rho$ in
$\game$, $\mu(\rho)$ for the play in $\Gamma_\game$.

\begin{lemma}
	For any policy $\sigma$ in $\game$ and for any policy $\sigma'$ in
	$\Gamma_\game$, if $\mu(\sigma) = \sigma'$ then $\wvalue(\sigma) =
	\wvalue'(\sigma')$.
\end{lemma}
\begin{proof}
	First, note that since $\game$ has observable rewards, then for all
	histories $h = a_0 \ldots o_n$ we have that for any two paths 
	$\rho = s_0 a_0 \ldots s_n, \rho' = s'_0 a_0 \ldots s'_n$ s.t.
	$\histfunc(\rho) = \histfunc(\rho') = h$ the following holds:
	\[
		\sum^{n-1}_{i=0} \discount^i \reward(s_i,a_i) = \sum^{n-1}_{i=0}
		\discount^i \reward(s'_i,a_i).
	\]
	Furthermore, by construction of $\Gamma_\game$ we also have that
	\[
		\sum^{n-1}_{i=0} \discount^i w(q_i,a_i,q_{i+1}) = \sum^{n-1}_{i=0}
		\discount^i \reward(s_i,a_i).
	\]
	Thus, for the result to follow, it suffices for us to show that for any
	policy $\sigma$ in $\game$ and corresponding $\sigma'$ in
	$\Gamma_\game$, if $\mu(\sigma) = \sigma'$ then $\mu$ is also bijective
	when restricted to plays consistent with $\sigma$ and $\sigma'$ in the
	respective structures.  We proceed by induction. Note that for any
	history $h$ in $\game$ with only one observation and consistent with
	$\sigma$ we have that $\mu(h)$ is consistent with $\sigma' =
	\mu(\sigma)$ since no choice has been made by the policies.
	Conversely, for any path $\pi$ in $\Gamma_\game$ with only one element,
	and consistent with $\sigma'$, $\mu^{-1}(\pi)$ is consistent with
	$\sigma = \mu^{-1}(\pi)$ for the same reason.
	Hence, for some $\ell$, $\mu$ is a bijective function from histories in
	$\game$ to paths in $\Gamma_\game$, all of length at most $\ell$.
	Consider a history $h = a_0 \ldots a_{\ell-1} o_\ell$ in $\game$
	consistent with $\sigma$ and let us write $\mu(h) =
	q_0 \ldots q_\ell$. By induction hypothesis, we know
	$\mu(a_0 \ldots o_{\ell-1}) =  q_0 \ldots q_{\ell-1}= \pi$ is
	consistent with $\sigma'$.  Observe that:
	\begin{itemize}
		\item $\sigma'(\pi) = \sigma(\mu^{-1}(\pi)) =
			\sigma(a_0 \ldots o_{\ell-1})$ and therefore $a_\ell
			\in \supp(\sigma'(\pi))$ since $h$ is consistent with
			$\sigma$;
		\item by definition of a history, there is some path
			$\chi = s_0 a_0 \ldots s_{\ell-1} a_{\ell-1} s_\ell)$ in
			$\game$ with $\histfunc(\chi) = h$; and
		\item by construction of $\Gamma_\game$ and definition of $\mu$
			we have that $s_{\ell-1} \in q_{\ell-1}$ and
			$(q_{\ell-1},a_{\ell-1},q_{\ell}) \in \Delta$.
	\end{itemize}
	It follows that $\mu(h)$ is also consistent with $\sigma'$. To show the
	other direction, we now take a path $\pi =
	q_0 a_0 \ldots a_{\ell-1} q_{\ell}$ in $\Gamma_\game$ consistent with
	$\sigma'$ and write $\mu^{-1}(\pi) = \ldots o_\ell$. It follows
	from inductive hypothesis that $\mu^{-1}(q_0 \ldots q_{\ell-1}) =
	\ldots o_{\ell-1} = h$ is consistent with $\sigma$. Since
	$\sigma(h) = \sigma'(\mu^{-1}(h))$, we have that $\supp(\sigma(h)) \ni
	a_{\ell -1}$. Also, for any $s \in q_\ell$ we have $o_\ell =
	\obsmap(s)$. Hence the claim holds and the result follows by induction.
\end{proof}

It follows from the above arguments that computing the worst-case value can be
done in exponential time for POMDPs with discounted sum and observable rewards.
This is, in fact, a tight complexity result. Indeed, safety and reachability
games with partial observation are EXP-hard~\citetrsec{cd10} even if the
objective is observable.  One can easily reduce either of them to a
discounted-sum objective in a POMDP by placing rewards or costs on target (or
unsafe) transitions (depending of the game we reduce from) and asking for
non-negative worst-case value. Therefore, deciding a threshold problem for the
worst-case value in POMDPs with discounted sum is EXP-complete.
\begin{theorem}
	The worst-case threshold problem for POMDPs with discounted sum and
	observable rewards is EXP-complete.
\end{theorem}

\section{Formal Proof of Proposition~\ref{prop:allowed}}
Assume we are given
a POMDP $\game=(\states,\act,\trans,\reward,\obs,\obsmap,\initd)$ with
observable rewards and we have
constructed the corresponding weighted arena $\Gamma_\game =
(Q,I,\act,\Delta,w)$.

Recall the statement says:\\
Let $\valbound\colon 2^\states \rightarrow \Rset$ be a function s.t.
$\valbound(B)\leq \fvalue(B)$ for each $B\in 2^\states$.
\begin{enumerate}[$(i)$]
	\item Then any policy $\sigma$ that is $\Psi$-safe for $t$ satisfies
		$\wvalue^{\pomdp}(\sigma)\geq t$.
	\item Moreover a policy $\pi$ is $\fvalue$-safe for $t$ if and only if
		$\wvalue^{\pomdp}(\pi)\geq t$.
\end{enumerate}

Item $(i)$ can easily be shown to hold by induction on the definition of a
strategy being $\Psi$-safe. For Item $(ii)$ we refer the reader to~\citetrsec{bmr14},
in which the authors show that, in discounted-sum games, playing $\fvalue$-safe
for $t$ is sufficient and necessary to obtain at least $t$. The result then
follows from the reduction from worst-case value in POMDPs with discounted sum
to discounted-sum games.
%
%

\section{Open Theoretical Problems}
The worst-case planning problem is open for general POMDPs. A lower bound for
the computational complexity of that problem would entail a lower bound for
universality of discounted sum automata, which is open~\citetrsec{cdh10}. In the
other direction, an upper bound (that is, an algorithm or any kind of
decidability result) would translate into an upper bound for the target
discounted sum problem~\citetrsec{BHO15:target-disc-sum}. The latter was shown to be
more general than some important open problems in mathematics and computer
science.

The exact GPO problem (i.e., not the $\epsilon$-approximation we achieve in this
work) is also open, even for fully-observable MDPs. Remark that if the
worst-case value threshold given is in fact the future value of the initial
state, then we could construct a sub-graph of choices which satisfy the equation
from system~\eqref{eq:bellman} and be sure that it is a complete representation
of the set of all policies achieving the optimal worst-case value. Hence, we
could optimize the expected value in that graph only and solve the GPO problem.
If the worst-case threshold is strictly lower, then this idea does not work.
Indeed, sub-optimal early choices might force later turns in the game to be
played optimally and vice versa as well. 

{\small
\bibliographystyletrsec{alpha}
\bibliographytrsec{bibliography-master,new}
}

\end{document}